\documentclass[conference]{IEEEtran}
\usepackage{amsmath,amssymb,amsfonts,amsthm}
\usepackage[hidelinks]{hyperref}
\usepackage{algorithm}
\usepackage{setspace}
\usepackage[noend]{algpseudocode}
\algrenewcommand\textproc{}
\usepackage{bm}
\usepackage{dsfont}
\usepackage{booktabs}
\usepackage{enumitem}
\usepackage{comment}
\usepackage{filecontents}
\usepackage{makecell}

\usepackage{tikz}
\usepackage{subfigure}
\usepackage{xcolor}
\usepackage{colortbl}
\usepackage{tikzscale}
\usetikzlibrary{trees,positioning,fit}
\usetikzlibrary{matrix,arrows.meta}
\usetikzlibrary{automata,decorations.markings,arrows}
\usetikzlibrary{shadings}
\usetikzlibrary{calc}
\graphicspath{{./png/}{../jpeg/}{./figures/}}
\DeclareGraphicsExtensions{.pdf,.tikz,.jpeg,.png,.eps}
\DeclareMathOperator*{\argmax}{arg\,\max}

\DeclareMathOperator*{\Ex}{\mathbb{E}}

\definecolor{green}{RGB}{0, 128, 0}
\newcommand{\ignore}[1]{{}}

\newcommand{\mybinom}[3][0.75]{\scalebox{#1}{$\tbinom{#2}{#3}$}}
\newcommand{\lexp}{\textbf{LExp}}
\newcommand{\adclick}{Ad-Clicks}
\newcommand{\coupon}{Coupon-Purchase}

\definecolor{dblue}{rgb}{0.2, 0.2, 0.9}

\makeatletter
  \def\my@tag@font{\normalsize}
  \def\maketag@@@#1{\hbox{\m@th\normalfont\my@tag@font#1}}
  \let\amsmath@eqref\eqref
  \renewcommand\eqref[1]{{\let\my@tag@font\relax\amsmath@eqref{#1}}}
\makeatother
\makeatletter
\DeclareRobustCommand{\rvdots}{
  \vbox{
    \baselineskip4\p@\lineskiplimit\z@
    \kern-\p@
    \hbox{.}\hbox{.}\hbox{.}
  }}
\makeatother

\allowdisplaybreaks
\newtheorem{theorem}{Theorem}

\begin{document}

\title{Multi-level Feedback Web Links Selection Problem: Learning and
  Optimization}

\author{Kechao Cai\textsuperscript{1}, Kun Chen\textsuperscript{2}, Longbo
  Huang\textsuperscript{2}, John C.S. 
  Lui\textsuperscript{1}, \\
  \textsuperscript{1}{Department of Computer Science \& Engineering,  The Chinese University of Hong Kong}\\
  \textsuperscript{2}{Institute for Interdisciplinary Information Sciences
    (IIIS), Tsinghua University}}
 
\maketitle

\begin{abstract}
  Selecting the right web links for a website is important because appropriate
  links not only can provide high attractiveness but can also increase the
  website's revenue. 
  In this work, we first show that web links have an intrinsic \emph{multi-level
    feedback structure}. 
  For example, consider a $2$-level feedback web link: the $1$st level feedback
  provides the Click-Through Rate (CTR) and the $2$nd level feedback provides
  the potential revenue, which collectively produce the compound $2$-level revenue.
  We consider the context-free links selection problem of
  selecting links for a homepage so as to maximize the total compound $2$-level
  revenue while keeping the total $1$st level feedback above a preset threshold. 
  We further generalize the problem to links with $n~(n\!\ge\!2)$-level feedback
  structure.
  To our best knowledge, we are the first to model the links selection problem
  as a constrained multi-armed bandit problem and design an effective links
  selection algorithm by learning the links' multi-level structure with provable
  \emph{sub-linear} regret and violation bounds. 
  We uncover the multi-level feedback structures of web links in two real-world
  datasets. 
  We also conduct extensive experiments on the datasets to compare our
  proposed \lexp{} algorithm with two state-of-the-art context-free bandit
  algorithms and  
  show that \lexp{} algorithm is the most effective in
  links selection while satisfying the constraint.
\end{abstract}

\section{Introduction}
\label{sec:intro}
Websites nowadays are offering many web links on their homepages to attract
users. 
For example, news websites such as Flipboard, CNN, and BBC constantly
update links to the news shown on their homepages to attract news readers. 
Online shopping websites such as Amazon and Taobao frequently refresh various
items on their homepages to attract customers for more purchase.

Each link shown on a homepage is intrinsically associated with a
\emph{``multi-level feedback structure''} which provides valuable information on
users' behaviors.
Specifically, based on the user-click information, the website can estimate the
probability (or Click-Through Rate (CTR)) that a user clicks that link, and we
refer to this as the \emph{$1$st level feedback}. 
Moreover, by tracking the behaviors of users after clicking the link (e.g.,
whether users will purchase products associated with that link), the website can
determine the revenue it can collect on that web page, we refer to this as the
\emph{$2$nd level feedback}. 
The \emph{compound $2$-level feedback} is a function of the $1$st level feedback
and the $2$nd level feedback. 
Naturally, the $1$st level feedback measures the \emph{attractiveness} of the
link, while the $2$nd level feedback measures the \emph{potential revenue} of
the link given that the link is clicked, and the compound $2$-level feedback
measures the \emph{compound revenue} of the link. 
In summary, for a given homepage, its total attractiveness is the sum of CTRs of
all links on that homepage, and its total compound revenue is the sum of the
compound revenue of all links on that homepage. 
Both the total attractiveness and the total compound revenue of a homepage are
important measures for investors to assess the value of a website~\cite{Kohavi2014KDD}.

Due to the limited screen size of mobile devices or the limited size of an
eye-catching area on a web page, homepages usually can only contain a finite
number of web links (e.g., Flipboard only shows 6 to 8 links for its users on
its homepage frame without sliding the frame.). 
Moreover, contextual information (e.g., the users' preferences) is not always
available due to visits from casual users, cold start~\cite{elahi2016survey} or
cookie blocking~\cite{meng2016trackmeornot}. 
Furthermore, a website with an unattractive homepage would be difficult to
attract investments.
In this case, it is important for website operators to consider the
\emph{context-free} web links selection problem: how to select a finite number
of links from a large pool of web links to show on its homepage so to maximize
the total compound revenue, while keeping the total attractiveness of the homepage above a preset
threshold?

The threshold constraint on the attractiveness of the homepage makes the above
links selection problem challenging. 
On the one hand, selecting those links with the highest CTRs ensures that the
attractiveness of the homepage is above the threshold, but it does not
necessarily guarantee the homepage will have high total compound revenue.
On the other hand, selecting links with the highest compound revenue cannot
guarantee that the total attractiveness of the homepage satisfies the threshold
constraint. 
Further complicating the links selection problem is the multi-level feedback
structures of web links, i.e., the CTRs ($1$st level feedback) and the potential
revenues ($2$nd level feedback), are \emph{unobservable} if the links are not
selected into the homepage.

To tackle this challenging links selection problem, we formulate a stochastic
constrained multi-armed bandit with multi-level rewards. 
Specifically, arms correspond to links in the pool, and each arm is associated
with a $1$st level reward corresponding to the $1$st level feedback (the CTR) of
a link, a $2$nd level reward corresponding to the $2$nd level feedback (the
potential revenue) of the link, and a compound $2$-level reward corresponding to
the compound $2$-level feedback of the same link.
Our objective is to select a finite number of links on the homepage so as to
maximize the cumulative compound $2$-level rewards (or minimizing the regret)
subject to a threshold constraint while learning/mining of links' multi-level
feedback structures.
To achieve this objective, we design a constrained bandit algorithm \lexp{},
which is not only effective in links selection, but also
achieves provable {\em sub-linear} regret and
violation bounds.

\noindent
{\bf Contributions:} (i)~We show that a web link is intrinsically
associated with a multi-level feedback structure. 
(ii)~To our best knowledge, we are the first to model the links selection
problem with multi-level feedback structures as a stochastic constrained bandit
problem (Sec.~\ref{sec:model}). 
(iii)~We design an bandit algorithm \lexp{} that selects $L$
arms from $K$ arms ($L\le K$) with provable {\em sub-linear} regret
and violation bounds (Sec.~\ref{sec:alg}).
(iv)~We show that \lexp{} is more effective in links selection
than two state-of-the-art context-free bandit algorithms,
\textsc{CUCB}~\cite{chen2013combinatorial} and
\textsc{Exp3.M}~\cite{uchiya2010algorithms}, via extensive experiments on two real-world datasets (Sec.~\ref{sec:experiments}).

 \section{Model}
\label{sec:model}

In this section, we first introduce the context-free web links selection problem
with a $2$-level feedback structure. 
Then we show how to formulate it as a stochastic constrained bandit problem, and
illustrate how it can model the links selection problem. 
Finally, we \emph{generalize} the links selection problem to links selection
problems with $n$-level feedback with $n\geq 2$.

\subsection{Bandit Formulation ({\small Constrained $2$-level Feedback})}
\label{subsec:bandit-with-2} 

Consider a website structure with a homepage frame and a pool of $K$ web pages,
$\mathcal{W}=\{w_{1},\ldots, w_{K}\}$.
Each web page $w_{i}\in\mathcal{W}$ is addressed by a URL link and so we have
$K$ links in total. 
The homepage frame can only accommodate up to $L \le K$ links. 
When we select the link associated with web page $w_{i}, 1\le i \le
K$, and put it into the homepage frame, we can observe the following
information when users browse the homepage:
\begin{enumerate}[leftmargin=0.5cm]
\item $A_{i} \geq 0$, the probability that a user clicks the link to
  $w_i$, which is also referred to as the click-through rate (CTR);
\item $B_{i} \geq 0$, the potential revenue received from the user who clicks the link and then purchases products (or browses ads) on the web page $w_{i}$.
\end{enumerate}
Therefore, for the link associated with web page $w_{i}$, the compound revenue
is $A_{i}B_{i}$, $1\le i \le K$. 
Our task is to select $L$ links from the pool of $K$ links for the homepage
frame. 
The objective is to maximize the total compound revenue of the selected $L$
links, subject to the constraint that the total CTR of these selected $L$ links
is greater than or equal to a preset threshold $h>0$.
Let $\mathcal{I}=\{i|w_{i}\in\mathcal{W}\}$ and $|\mathcal{I}|=L$ be the set of 
indices of any $L$ links.
Denote the feasible set of the above links selection problem as $\mathcal{S}$,
which contains all possible subsets of indices of any $L$ links such that
satisfy the total CTR requirement $h$. 
Specifically, the optimal set of the $L$ links for the described links selection
problem is the solution to the following constrained knapsack problem,
\begin{equation}
\label{eq:problem}
\begin{aligned}
&\!\argmax_{\mathcal{I} \in \mathcal{S}}\sum\nolimits_{i\in\mathcal{I}}A_{i}B_{i},\\
&\mathcal{S} =\big\{ \mathcal{I}=\{i|w_{i}\in\mathcal{W}\} \big| |\mathcal{I}|=L, \sum\nolimits_{i\in\mathcal{I}}A_{i}\ge h\big\}.\\
\end{aligned}
\end{equation}

Problem (\ref{eq:problem}) is known to be NP-hard~\cite{korte1982existence}.
To tackle this problem, we relax (\ref{eq:problem}) to a probabilistic linear programming problem
(\ref{eq:problem-relaxed}) as follows,
\begin{equation}
\label{eq:problem-relaxed}
\begin{aligned}
&\!\argmax_{\bm{x}\in \mathcal{S}'}\sum\nolimits_{i=1}^{K} x_{i}A_{i}B_{i},\\
&\mathcal{S}' =\big\{\bm{x}\in[0,1]^{K}\big|\sum\nolimits_{i=1}^{K} x_{i}A_{i} \ge h, \sum\nolimits_{i=1}^{K} x_{i}=L \big\},\\
\end{aligned}
\end{equation}
where $\bm{x}=(x_{1},\dots,x_{i},\dots,x_{K})$ and $x_{i}$ represents the
probability of selecting the web page $w_{i},1\le i\le K$. 
Note that problem (\ref{eq:problem-relaxed}) is still non-trivial to solve
because $A_{i}$ and $B_{i}$ are only {\em observable} if the web page $w_i$ is
{\em selected} to the homepage frame. 
If $w_i$ is not selected, one {\em cannot} observe $A_i$ or $B_i$. 

To answer problem~\eqref{eq:problem-relaxed}, we formulate the links selection
problem as a stochastic constrained multi-armed bandit problem with $2$-level
rewards and design a constrained bandit
algorithm. 
Formally, let $\mathcal{K}=\{1,\ldots, K\}$ denote the set of arms, where each
arm corresponds to a specific link to a web page in $\mathcal{W}$. Each arm
$i\in \mathcal{K}$ is associated with two \emph{unknown} random processes,
$A_{i}(t)$ and $B_{i}(t)$, $t=1,\dots, T$. Specifically, $A_{i}(t)$
characterizes the arm $i$'s \emph{$1$st level reward} which corresponds to link
$i$'s $1$st level feedback (CTR), and $B_{i}(t)$ characterizes arm $i$'s
\emph{$2$nd level reward} which corresponds to link $i$'s $2$nd level feedback
(potential revenue) that can be collected from $w_{i}$. We assume that
$A_{i}(t)$ are stationary and independent across $i$, and the probability
distribution of $A_{i}(t)$ has a finite support. 
As for $B_{i}(t)$, they \emph{are not necessarily stationary} due to the
heterogeneity of users but are bounded across $i$.
Without loss of generality, we normalize $A_{i}(t)\in [0,1]$ and $B_{i}(t)\in
[0,1]$. 
We also assume that $A_{i}(t)$ is independent of $B_{i}(t)$ for $i\in
\mathcal{K}$, $t\ge 1$. 
Note that this assumption is reasonable as we have observed and validated that
different level feedbacks of links in the multi-level feedback structure do not
have strong correlations in the real-world datasets (please refer to
Sec.~\ref{sec:experiments}).

The stationary random process $A_{i}(t)$, is assumed to have \emph{unknown} mean
$a_{i}=\Ex[A_{i}(t)]$ for $1\le i \le K$. 
Let $\bm{a} = (a_{1},\ldots, a_{K})$.\footnote{All vectors defined in this paper
  are column vectors.} 
Let $\bm{a}_{t} = (a_{1}^{t},\ldots, a_{K}^{t})$ and $\bm{b}_{t} =
(b_{1}^{t},\ldots, b_{K}^{t})$ denote the realization vectors for the random
processes $A_{i}(t)$ and $B_{i}(t)$, respectively for $1\le i \le K$.
Let $\bm{x}_{t}=(x_{1}^{t},\ldots,x_{i}^{t},\ldots, x_{K}^{t})$ be the
\emph{probabilistic selection vector} of the $K$ arms at time $t$, where
$x_{i}^{t}\in[0,1]$ is the probability of selecting of the arm $i$ at time $t$.
The number of selected arms is $L$ at each time $t$, i.e.,
$\bm{1}^{\intercal}\bm{x}_{t} = L$, where $\bm{1}=(1,\ldots,1)$ is the one
vector. 
At time $t$, a set of $L\le K$ arms $\mathcal{I}_{t}\in\mathcal{K}$ is selected
via a dependent rounding procedure~\cite{gandhi2006dependent}, which guarantees
the probability that $i\in\mathcal{I}_{t}$ is $x_{i}^{t}$ at time $t$ (see
Sec.~\ref{sec:alg}). 
For each arm $i\in \mathcal{I}_{t}$, the algorithm observes a \emph{$1$st level
  reward} $a_{i}^{t}$ generated by $A_{i}(t)$ as well as a \emph{$2$nd level
  reward} $b_{i}^{t}$ generated by $B_{i}(t)$, and receives a \emph{compound
  $2$-level} reward. 
Specifically, the compound $2$-level reward, $g_{i}^{t}$, of an arm $i$ at time
$t$ is generated by the random process $G_{i}(t)=A_{i}(t) B_{i}(t)$. 
Let $g_{i}^{t}=a_{i}^{t}b_{i}^{t},1\le i\le K$ and $\bm{g}_{t}=(g_{1}^{t}, \ldots, g_{i}^{t}, \ldots, g_{K}^{t})$.
In addition, there is a preset threshold $h>0$ such that the average of the sum
of the $1$st level rewards needs to be above this threshold, i.e.,
$\bm{a}^{\intercal}\Ex[\bm{x}_{t}]\ge h$.\footnote{If $h=0$, the problem
  is equivalent to the classic unconstrained multiple play multi-armed
  bandit problem (MP-MAB)~\cite{anantharam1987asymptotically}.} 
At time $t$, the expected total compound $2$-level reward is
$\Ex[\sum_{t}\bm{g_{t}}^{\intercal}\bm{x}_{t}]$ with the probabilistic selection
vector $\bm{x}_{t}$, $t=1,\dots,T$.

Our objective is to design an algorithm to choose the selection vectors
$\bm{x}_{t}$ for $t=1,\ldots,T$ such that the \emph{regret}, which is also
referred to as loss compared with the oracle $\max_{\bm{a}^{\intercal}\bm{x}\ge
  h}\sum\nolimits_{t=1}^{T}\bm{g}_{t}^{\intercal}\bm{x}$, is minimized.
Specifically, the regret for an algorithm $\pi$ is,
\begin{equation}
\label{eq:regret-def}
\begin{aligned}
\textmd{Reg}_{\pi}(T) = \max_{\bm{a}^{\intercal}\bm{x}\ge h}\sum\nolimits_{t=1}^{T}\bm{g}_{t}^{\intercal}\bm{x} -
                        \Ex\big[ \sum\nolimits_{t=1}^{T}\bm{g}_{t}^{\intercal}\bm{x}^{\pi}_{t} \big],
\end{aligned}
\end{equation}
where $\bm{x}^{\pi}_{t}$ is the probabilistic selection vector calculated by the
algorithm $\pi$ at time $t$.
Note that $\bm{x}^{\pi}_{t}$ may violate the constraint initially especially
when we have little information about the arms. 
To measure the overall violations of the constraint at time $T$, the
\emph{violation} of the algorithm $\pi$ is defined as,
\begin{equation}
\label{eq:violation-def}
\begin{aligned}
\textmd{Vio}_{\pi}(T) = \Ex\big[\sum\nolimits_{t=1}^{T}(h-\bm{a}^{\intercal}\bm{x}^{\pi}_{t})\big]_{+},
\end{aligned}
\end{equation}
where $[x]_{+}=\max(x,0)$. 
Note that if the regret and violation of an algorithm are linear, the
algorithm is \emph{not learning}. 
A simple example of such algorithms is the uniform arm selection algorithm where
any $L$ arms are selected with equal probability. 
Such a random policy would result in both linear regret and linear violation as
there is a constant loss compared to the optimal policy at each time $t$.

\subsection{Generalization to $n$-level Feedback, where $n\ge2$}
\label{subsec:bandit-with-n}
We can further extend the constrained multi-armed bandit model with $2$-level
reward to the constrained multi-armed bandit model with $n$-level ($n\ge 2$)
reward, and this allows us to model links selection problem with $n$-level
feedback structure. 
Specifically, 
we can take each web page $w_{i}\in
\mathcal{W}$ as a pseudo homepage frame. 
For the pseudo homepage frame, there is a pool of web pages $\mathcal{W}',
|\mathcal{W}|=K'$. 
Then we consider selecting a subset of $L'$ links $\mathcal{I}'$,
$|\mathcal{I}'|=L'$ (that each links to a web page in $\mathcal{W}'$) for the
pseudo home page frame, with the constraint that the total CTR on the pseudo
homepage frame is above the threshold $h'$. 
Formally, for each web page $w_{i}\in \mathcal{W}$, we consider the potential revenue of
$B_{i}$ in a much more precise way, i.e., $B_{i}=\sum_{j\in
  \mathcal{I'}}A'_{j}B'_{j}$, where $A'_{j}$ is the CTR of the link associated
with the web page $w'_{j}\in \mathcal{I}'$ and $B'_{j}$ is the potential revenue collected
from the web page $w'_{j}\in \mathcal{I}'$. 
As such, we extend the links selection problem with $2$-level feedback ($A_{i}$
and $B_{i}$ where $i\in\mathcal{I}$) to a problem with $3$-level feedback
And similarly, we can further extend the problem to the problems with $n$-level
feedback structure where $n\ge 2$.

 \section{Algorithm \& Analysis}
\label{sec:alg}
In this section, we first elaborate the design of our constrained bandit
algorithm \lexp{} (which stands for ``$\mybinom{K}{L}$-Lagrangian Exponential
weights'') and present the algorithmic details. 
Then we provide both regret and violation analysis and show that our algorithm
has the attractive property of being {\em sub-linear} in both regret and
violation.
\setlength{\textfloatsep}{11pt}
\begin{algorithm}[t]
  \caption{\lexp{} ($\gamma, \delta$)}
  \label{alg:lexp}
  \begin{algorithmic}[1]
    \Statex {\textbf{Init:} $\bm{\eta}^{1}=\bm{1},\lambda_{1}=0, h>0, \beta=(1/L-\gamma/K)/(1-\gamma)$}
    \For{$t= 1,\dots, T$}
    \State \label{alg:lexp:init-set}{${\mathcal{A}}_t=\emptyset$,\,
$\mathcal{I}_{t}=\emptyset$, $\alpha_{t}=0$.}
    \If {$\max_{i\in \mathcal{K}}\eta_{i}^{t}\ge \beta\sum_{i=1}^{K}\eta_{i}^{t}$}\label{alg:lexp:if-max}
    \State \label{alg:lexp:alphat}{Find $\alpha_{t}$ such that
            $${\alpha_{t}}/\big({\textstyle\sum\nolimits_{i=1,\eta_{i}^{t}\ge \alpha_{t}}^{K}\alpha_{t}+ \textstyle\sum\nolimits_{i=1,\eta_{i}^{t}<\alpha_{t}}^{K}\eta_{i}^{t}}\big)=\beta$$}
    \State \label{alg:lexp:Ahigh}{${\mathcal{A}}_{t} = \{i: \eta_{i}^{t}\ge \alpha_{t}\}$}
    \EndIf
    \For{$i=1,\ldots,K$} \label{alg:lexp:eta-vector}
\vspace{-4pt}
  $$\tilde{\eta}_{i}^{t} =\alpha_{t}   \textmd{ if } i\in {\mathcal{A}}_t;
\textmd{ otherwise, } 
\tilde{\eta}_{i}^{t} = \eta_{i}^{t}$$
\EndFor
\vspace{-6pt}
 \For{$i=1,\ldots,K$} \label{alg:lexp:x-vectorx}
\vspace{-4pt}
 \begin{equation*}
 \tilde{x}_{i}^{t}= L[(1-\gamma)\tilde{\eta}_{i}^{t}/\textstyle\sum\nolimits_{i=1}^{K}\tilde{\eta}_{i}^{t} + \gamma/K]
 \end{equation*}
 \EndFor
\vspace{-4pt}
    \State \label{alg:lexp:depround}{$\mathcal{I}_{t}=\text{DependentRounding}(L, \tilde{\bm{x}}_{t})$}
    \For {$i\in \mathcal{I}_{t}$} \label{alg:lexp:receive-ab} receive $a_{i}^{t}$, and $b_{i}^{t}$
\EndFor   
    \For {$i=1,\ldots,K$} \label{alg:lexp:estimate-abg}
\vspace{-4pt}
\begin{align*}
 \hat a_{i}^{t}  = {a_{i}^{t}}/{ \tilde{x}_{i}^{t}}{\textstyle \mathds{1}(i\in \mathcal{I}_{t})},\,\,
 \hat g_{i}^{t}  = {a_{i}^{t}b_{i}^{t}}/{ \tilde{x}_{i}^{t}}{\textstyle \mathds{1}(i\in \mathcal{I}_{t})}
\end{align*}
\EndFor
\vspace{-4pt}
    \For {$i=1,\ldots, K$}\label{alg:lexp:eta-update}
\vspace{-4pt}
\begin{equation*}
\eta_{i}^{t+1} =
\begin{cases}
  \eta_{i}^{t}                                          & \textmd{if } i\in {\mathcal{A}}_{t}; \\
  \eta_{i}^{t}\exp[\zeta (\hat g_{i}^{t}
                         + \lambda_{t} \hat a_{i}^{t})] & \textmd{if } i\notin {\mathcal{A}}_{t}
\end{cases}
\end{equation*}
\EndFor
\vspace{-4pt}
    \State \label{alg:lexp:lambda-update}{$\lambda_{t+1} = [(1-\delta\zeta)\lambda_{t}- \zeta(\frac{\hat{\bm{a}}_{t}^{\intercal}\tilde{\bm{x}}_{t}}{1-\gamma}-h)]_{+}$}
    \EndFor
\Function{DependentRounding}{$L, \bm{x}$}\label{alg:lexp:func-dr}
\While  {exist $x_{i}\in(0,1)$}\label{alg:depround:while}
\State \label{alg:depround:ij}{Find $i,j, i\neq j$, such that $x_{i,j}\in(0,1)$}
\State \label{alg:depround:complement}{$p=\min\{1-x_{i},x_{j}\}$,
$q=\min\{x_{i},1-x_{j}\}$}
\State \label{alg:depround:xixj}{
  $(x_{i},x_{j})=
\begin{cases}
(x_{i}+p, x_{j}-p)  \text{ with prob. } \frac{q}{p+q}; \\
(x_{i}-q, x_{j}+q)  \text{ with prob. } \frac{p}{p+q}.
\end{cases}$}
\EndWhile
\Return $\mathcal{I}=\{i\,|\,x_{i}=1, 1\le i\le K\}$\label{alg:depround:return}
\EndFunction
  \end{algorithmic}
\end{algorithm}
\subsection{Constrained Bandit Algorithm}
\label{sec:contr-band-algor}

The unique challenge for our
algorithmic design is to balance between maximizing the compound multi-level
rewards (or minimizing the regret) and at the same time, satisfying the
threshold constraint. 
To address this challenge, we incorporate the theory of Lagrange method in
constrained optimization into the design of \lexp{}. 
We consider minimizing a modified regret function that includes the violation
with an \emph{adjustable} penalty coefficient that increases the regret when
there is any non-zero violation. 
Specifically, \lexp{} introduces a sub-linear bound for the Lagrange
function of $\textmd{Reg}_{\pi}(T)$ and $\textmd{Vio}_{\pi}(T)$ in the following
structure,
\begin{equation}
\label{eq:lagrange}
 \textmd{Reg}_{\pi}(T) +\rho(T) \textmd{Vio}^{2}_{\pi}(T) \le T^{1-\theta}, 0<\theta\le 1,
\end{equation} 
where $\rho(T)$ plays the role of a Lagrange multiplier. 
From (\ref{eq:lagrange}), we can derive a bound for $\textmd{Reg}_{\pi}(T)$ and
a bound for $\textmd{Vio}_{\pi}(T)$ as follows:
\begin{equation}
\label{eq:vio-bound}
\!\!\!\!\textmd{Reg}_{\pi}(T)\! \le\! O(T^{1-\theta}),\!
\textmd{Vio}_{\pi}(T)\!\le\!\!\sqrt{{O(T^{1-\theta}\! +\! LT)}/{\rho(T)}},
\end{equation}
where the bound for $\text{Vio}_{\pi}(T)$ in (\ref{eq:vio-bound}) is for the
fact that $-\textmd{Reg}_{\pi}(T)\le O(LT)$ for any algorithm $\pi$. 
Thus, with properly chosen algorithm parameter $\rho(T)$, both the regret and violation can
be bounded by sub-linear functions of $T$.

The details of \lexp{} are shown in Algorithm~\ref{alg:lexp}.
In particular, \lexp{} maintains a weight vector
$\bm{\eta}^{t}$ at time
$t$, which is used to calculate the probabilistic selection vector
$\tilde{\bm{x}}_{t}$ (line~\ref{alg:lexp:if-max} to
line~\ref{alg:lexp:x-vectorx}). 
Specially, line~\ref{alg:lexp:if-max} to line~\ref{alg:lexp:eta-vector} ensure that the
probabilities in $\tilde{\bm{x}}_{t}$ are less than or equal to 1.
At line~\ref{alg:lexp:depround}, we deploy the dependent rounding function
(line~\ref{alg:lexp:func-dr} to line~\ref{alg:depround:return}) to
select $L$ arms using the calculated $\tilde{\bm{x}}_{t}$. 
At line~\ref{alg:lexp:receive-ab}, the algorithm obtains the rewards
$a_{i}^{t}$ and $b_{i}^{t}$, and then 
gives \emph{unbiased} estimates
of $\hat a_{i}^{t}$ and $\hat g_{i}^{t}$ at line~\ref{alg:lexp:estimate-abg}. 
Specifically, the $1$st level reward $\hat a_{i}^{t}$, and the compound
$2$-level reward $\hat g_{i}^{t}$ are estimated by $a_{i}^{t}/\tilde{x}_{i}^{t}$, and
$a_{i}^{t}b_{i}^t/\tilde{x}_{i}^{t}$, respectively, such that $\Ex[\hat
a_{i}^{t}]=a_{i}^{t}$, and $\Ex[\hat g_{i}^{t}]=a_{i}^{t}b_{i}^{t}$. 
Finally, the weight vector $\bm{\eta}_{t}$ and the Lagrange multiplier
$\lambda_{t}$ are updated (line~\ref{alg:lexp:eta-update} and
line~\ref{alg:lexp:lambda-update}) using previous estimations. 

\subsection{Regret and Violation Analysis}
\label{sec:regr-viol-analys}

\begin{theorem}\label{theorem:bounds}
  Let $\zeta=\frac{\gamma\delta L}{(\delta+L)K}$,
  $\gamma=\Theta(T^{-\frac{1}{3}})$ and $\delta=\Theta(T^{-\frac{1}{3}})$ that
  satisfy $\delta\ge\frac{4(e-2)\gamma L}{1-\gamma}-L$. 
  By running the \emph{\textbf{LExp}} algorithm $\tilde{\pi}$,
  we achieve sub-linear bounds for both the regret in \eqref{eq:regret-def} and
  violation in~\eqref{eq:violation-def} as follows:
\begin{equation*}
    \emph{\text{Reg}}_{\tilde{\pi}}(T) \!\le\! O(LK\ln(K)T^{\frac23}) \text{ and }
    \emph{\text{Vio}}_{\tilde{\pi}}(T) \!\le\! O(L^\frac{1}{2}K^\frac{1}{2}T^\frac{5}{6}).
  \end{equation*}
\end{theorem}
  
\begin{proof}
  From line~\ref{alg:lexp:lambda-update} of the algorithm, we have:
$\textstyle\lambda_{t+1}\!=\!\big[(1-\delta\zeta)\lambda_{t}- \zeta(\frac{\hat{\bm{a}}_{t}^{\intercal}\tilde{\bm{x}}_{t}}{1-\gamma}-h)\big]_{+}\le \big[(1-\delta\zeta)\lambda_{t}+\zeta h)\big]_{+}$.
By induction on $\lambda_t$, we can obtain $\lambda_t\le\frac{h}{\delta}$. 
Let $\Phi_{t}=\sum_{i=1}^{K}\eta_{i}^{t}$ and $\tilde{\Phi}_{t}=
\sum_{i=1}^{K}\tilde{\eta}_{i}^{t}$. 
Define $\bm r_{t}=\bm{g}_{t}+\lambda_{t}\bm a_{t}$ and $\hat{\bm r}_{t}=\hat{\bm
  g}_{t}+ \lambda_{t} \hat{\bm a}_{t}$. 
Let $\bm x$ be an arbitrary probabilistic selection vector which satisfies
$x_{i}\in[0,1]$, $\bm{1}^{\intercal}\bm{x}_{t}= L$ and
$\bm{a}^{\intercal}\bm{x}\ge h$. 
We know that
\begin{align}
&\sum\limits_{t=1}^T\ln\frac{\Phi_{t+1}}{\Phi_t}=\ln\frac{\Phi_{T+1}}{\Phi_1}=\ln(\sum\limits_{i=1}^K\eta_i^{T+1})-\ln K \nonumber\\
&\ge\ln(\sum_{i=1}^K x_i\eta_i^{T+1})-\ln K  \nonumber
\ge\sum_{i=1}^K\frac{x_i}{L}\sum_{t:i\notin \mathcal A_t}\zeta\hat{r}_{i}^{t}-\ln\frac{K}{L}  \nonumber\\
&=\frac{\zeta}{L}\sum\limits_{i=1}^Kx_i\sum\limits_{t:i\notin \mathcal A_t}\hat{r}_{i}^{t}-\ln\frac{K}{L}. \label{eq:lexp:primal:lowerbound}
\end{align}
As $\zeta=\frac{\gamma\delta L}{(\delta+L)K}$ and
$\lambda_t\le\frac{h}{\delta}$, we have $\zeta \hat r_{i}^{t}\le1$. Therefore,
\begin{align}
&\frac{\Phi_{t+1}}{\Phi_t}=\!\!\!\!\!\sum_{i\in \mathcal{K}/\mathcal A_{t}}\!\!\!\frac{\eta_i^{t+1}}{\Phi_t}\!+\!\!\sum_{i\in \mathcal A_t}\!\frac{\eta_i^{t+1}}{\Phi_t}\!=\!\!\!\!\!\!\sum_{i\in \mathcal{K}/\mathcal A_t}\!\!\!\frac{\eta_i^{t}}{\Phi_t}\exp(\zeta\hat{r}_i^t)\!+\!\!\sum_{i\in \mathcal A_t}\frac{\eta_i^{t}}{\Phi_t} \nonumber \\
&\le\sum_{i\in \mathcal{K}/\mathcal A_t}\frac{\eta_i^{t}}{\Phi_t}[1+\zeta\hat{r}_i^t+(e-2)\zeta^2(\hat{r}_i^t)^2]+\sum_{i\in \mathcal A_t}\frac{\eta_i^{t}}{\Phi_t} \label{eq:lexp:expfunction:upperbound} \\
&=1+\frac{\tilde{\Phi}_{t}}{\Phi_t}\sum_{i\in \mathcal{K}/\mathcal A_t}\frac{\eta_i^{t}}{\tilde{\Phi}_t}\left[\zeta\hat{r}_i^t+(e-2)\zeta^2(\hat{r}_i^t)^2\right] \nonumber \\
&\le 1+\sum_{i\in \mathcal{K}/\mathcal A_t}\frac{\tilde{x}_i^t/L-\gamma/K}{1-\gamma}\left[\zeta\hat{r}_i^t+(e-2)\zeta^2(\hat{r}_i^t)^2\right] \nonumber \\
&\le 1\!+\!\frac{\zeta}{L(1-\gamma)}\sum_{i\in \mathcal{K}/\mathcal A_{t}}\!\!\tilde{x}_i^t\hat{r}_i^t+\frac{(e-2)\zeta^2}{L(1-\gamma)}\sum_{i\in \mathcal{K}/\mathcal A_t}\tilde{x}_i^t(\hat{r}_i^t)^2 \nonumber \\
&\le 1\!+\!\frac{\zeta}{L(1-\gamma)}\sum_{i\in \mathcal{K}/\mathcal A_{t}}\!\!\tilde{x}_i^t\hat{r}_i^t+\frac{(e-2)\zeta^2}{L(1-\gamma)}\sum_{i=1}^{K}(1+\lambda_t)\hat{r}_i^t \label{eq:lexp:estimation-square}.
\end{align}
Inequality \eqref{eq:lexp:expfunction:upperbound} holds because $e^y\le
1+y+(e-2)y^2$ for $y\le 1$, and inequality \eqref{eq:lexp:estimation-square} uses
the fact that $\tilde{x}_i^t\hat{r}_i^t=r_i^t\le 1+\lambda_t$ for $i\in I_{t}$
and $\tilde{x}_i^t\hat{r}_i^t=0$ for $i\notin I_{t}$. 
Since $\ln(1+y)\le y$ for $y\ge 0$, we can get
\begin{equation*}
\ln\frac{\Phi_{t+1}}{\Phi_t}\le \frac{\zeta}{L(1-\gamma)}\sum_{i\in \mathcal{K}/\mathcal A_{t}}\!\!\!\tilde{x}_i^t\hat{r}_i^t+\frac{(e-2)\zeta^2}{L(1-\gamma)}\sum_{i=1}^{K}(1+\lambda_t)\hat{r}_i^t.
\end{equation*}
Then using \eqref{eq:lexp:primal:lowerbound}, it follows that
\begin{align*}
\frac{\zeta}{L}\sum_{i=1}^K x_i \!\!\sum_{t:i\notin \mathcal A_t}& \hat{r}_{i}^{t}-\ln\frac{K}{L} \le\frac{\zeta}{L(1-\gamma)}\sum_{t=1}^{T}\sum_{i\in \mathcal{K}/\mathcal A_{t}}\tilde{x}_i^t\hat{r}_i^t \\
&+\frac{(e-2)\zeta^2}{L(1-\gamma)}\sum_{t=1}^{T}\sum_{i=1}^{K}(1+\lambda_t)\hat{r}_i^t.
\end{align*}
As $\tilde{x}_i^t=1$ for $i\in \mathcal A_t$, and $\sum_{i=1}^Kx_i\sum_{t:i\in \mathcal A_t}\hat{r}_i^t \le \frac{1}{1-\gamma}\sum_{t=1}^T\sum_{i\in \mathcal A_{t}}\hat{r}_i^t$ trivially holds, we have
\begin{equation*}
\sum_{t=1}^T\hat{\bm r}_{t}^{\intercal}\bm{x}-
\frac{L}{\zeta}\ln\frac{K}{L}\le\frac{\sum_{t=1}^T\hat{\bm{r}}_{t}^{\intercal}\tilde{\bm{x}}_{t}}{1-\gamma}+\frac{(e-2)\zeta}{1-\gamma}\sum_{t=1}^T\sum_{i=1}^{K}(1+\lambda_t)\hat{r}_{i}^{t}.
\end{equation*}
Taking expectation on both sides, we have
\begin{align}
&\Ex\Big[\sum\nolimits_{t=1}^T\hat{\bm r}_{t}^{\intercal}\bm{x}-\frac{1}{1-\gamma}\sum\nolimits_{t=1}^T\hat{\bm{r}}_{t}^{\intercal}\tilde{\bm{x}}_{t}\Big] \nonumber \\
&\le \frac{L}{\zeta}\ln\frac{K}{L}+\frac{(e-2)\zeta}{1-\gamma}\sum\nolimits_{t=1}^T\Ex\Big[\sum\nolimits_{i=1}^{K}(1+\lambda_t)\hat{r}_{i}^{t}\Big] \nonumber \\
&\le \frac{L}{\zeta}\ln\frac{K}{L}+\frac{2(e-2)\zeta K}{1-\gamma}T+ \frac{2(e-2)\zeta K}{1-\gamma}\sum\nolimits_{t=1}^T\lambda_t^{2}, \label{eq:lexp:primal}
\end{align}
where \eqref{eq:lexp:primal} is from the inequality
$\Ex[\sum\nolimits_{i=1}^{K}(1+\lambda_t)\hat{r}_{i}^{t}]=\sum_{i=1}^{K}(1+\lambda_{t})(g_{i}^{t}+
\lambda_{t} a_{i}^{t})\le2K+2K\lambda_{t}^{2}$.
Next, we define a series of functions
$f_{t}(\lambda)=\frac{\delta}{2}\lambda^{2}+\lambda(\frac{1}{1-\gamma}\hat{\bm{a}}_{t}^{\intercal}\tilde{\bm{x}}_{t}-h),
t=1,\dots,T$, and we have $\lambda_{t+1}=\left[\lambda_t-\zeta\nabla
  f_t(\lambda_t)\right]_+$. 
It is clear that $f_{t}(\cdot)$ is a convex function for all $t$. 
Thus, for an arbitrary $\lambda$, we have
\begin{align*}
&(\lambda_{t+1}-\lambda)^2=(\left[\lambda_t-\zeta\nabla f_t(\lambda_t)\right]_+-\lambda)^2 \\
&\le(\lambda_t-\lambda)^2\!+\!\zeta^2(\delta\lambda_t-h+\frac{\hat{\bm{a}}_{t}^{\intercal}\tilde{\bm{x}}_{t}}{1-\gamma})^{2}\!-\!2\zeta(\lambda_t-\lambda)\nabla f_t(\lambda_t) \\
&\le(\lambda_t-\lambda)^2+2\zeta^2h^2+2\zeta^2\frac{(\hat{\bm{a}}_{t}^{\intercal}\tilde{\bm{x}}_{t})^2}{(1-\gamma)^2}+2\zeta[f_t(\lambda)-f_t(\lambda_t)].
\end{align*}
Let $\Delta = [(\lambda_t-\lambda)^2-(\lambda_{t+1}-\lambda)^2]/(2\zeta)+\zeta L^{2}$. We have,
\begin{align*}
&f_t(\lambda_t)-f_t(\lambda)\le \Delta+\frac{\zeta(\hat{\bm{a}}_{t}^{\intercal}\tilde{\bm{x}}_{t})^2}{(1-\gamma)^2} 
\!=\!\Delta+\frac{\zeta L^{2} (\frac{1}{L}\hat{\bm{a}}_{t}^{\intercal}\tilde{\bm{x}}_{t})^2}{(1-\gamma)^2}
\\
&\le\Delta+\frac{\zeta
 L^{2}}{(1-\gamma)^2}\frac{1}{L}\sum_{i=1}^{K}(\tilde{x}_{i}^{t}\hat{a}_i^{t})^2
\le\Delta+\frac{\zeta L}{(1-\gamma)^2}\sum\limits_{i=1}^K{a}_i^t.
\end{align*}
Taking expectation over $\sum_{t=1}^{T}[f_t(\lambda_t)-f_t(\lambda)]$, we have
\begin{align}
&\Ex\big[\frac{\delta}{2}\sum\nolimits_{t=1}^T\lambda_t^2-\frac{\delta}{2}\lambda^2 T+\sum\nolimits_{t=1}^T\lambda_t(\frac{\hat{\bm{a}}_{t}^{\intercal}\tilde{\bm{x}}_{t}}{1-\gamma}-h) \nonumber\\
&-\lambda\sum\nolimits_{t=1}^T(\frac{\hat{\bm{a}}_{t}^{\intercal}\tilde{\bm{x}}_{t}}{1-\gamma}-h)\big]\le\frac{\lambda^2}{2\zeta}+\zeta L^2T+\frac{\zeta LK}{(1-\gamma)^2}T. \label{eq:lexp:dual}
\end{align}
Combining \eqref{eq:lexp:primal} and \eqref{eq:lexp:dual}, we have,
\begin{align*}
&\sum\nolimits_{t=1}^T\bm{g}_{t}^{\intercal}\bm{x}-\frac{\Ex[\sum\nolimits_{t=1}^{T}\bm{g}_{t}^{\intercal}\tilde{\bm{x}}_{t}]}{1-\gamma}+\Ex\big[-(\frac{\delta T}{2}+\frac{1}{2\zeta})\lambda^2\\
& \quad + \lambda\sum\nolimits_{t=1}^T(h-\frac{\bm{a}^{\intercal}\tilde{\bm{x}}_{t}}{1-\gamma})\big]\le \frac{L}{\zeta}\ln\frac{K}{L}+\frac{2(e-2)\zeta KT}{1-\gamma}\\
&+\zeta L^2 T+\frac{\zeta LKT}{(1-\gamma)^2}+ (\frac{2(e-2)\zeta K}{1-\gamma}-\frac{\delta}{2})\sum\nolimits_{t=1}^T\lambda_t^2\\
&+\Ex[\sum\nolimits_{t=1}^T\lambda_t(h-\bm{a}^{\intercal}\bm{x})].
\end{align*}
Since $\zeta=\frac{\gamma\delta L}{(\delta+L)K}$ and
$\delta\ge\frac{4(e-2)\gamma L}{1-\gamma}-L$, we have $\frac{2(e-2)\zeta
  K}{1-\gamma}\le\frac{\delta}{2}$. As $\bm{a}^{\intercal}\bm{x}\ge h$, we have
\begin{align*}
&(1-\gamma)\sum\nolimits_{t=1}^T\bm{g}_{t}^{\intercal}\bm{x}-\Ex\big[\sum\nolimits_{t=1}^{T}\bm{g}_{t}^{\intercal}\tilde{\bm{x}}_{t}\big]\\
&\quad+\Ex\big[\lambda\sum\nolimits_{t=1}^T((1-\gamma)h-\bm{a}^{\intercal}\tilde{\bm{x}}_{t}) -(\frac{\delta T}{2}+\frac{1}{2\zeta})\lambda^2\big]\\
&\le \frac{L}{\zeta}\ln\frac{K}{L}+2(e-2)\zeta KT+\zeta L^2 T+\frac{\zeta LKT}{1-\gamma}.
\end{align*}
Let
$\lambda=\frac{\sum_{t=1}^T((1-\gamma)h-\bm{a}^{\intercal}\tilde{\bm{x}}_{t})}{\delta
  T+1/\zeta}$. Maximize over $\bm{x}$ and we have,
\begin{align*}
&\max_{\bm{a}^{\intercal}\bm{x}\ge h}\sum\nolimits_{t=1}^T\bm{g}_{t}^{\intercal}\bm{x}-\Ex\big[\sum\nolimits_{t=1}^{T}\bm{g}_{t}^{\intercal}\tilde{\bm{x}}_{t}\big]\\
&\quad\quad+\Ex\Big\{\frac{\big[\sum\nolimits_{t=1}^T((1-\gamma)h-\bm{a}^{\intercal}\tilde{\bm{x}}_{t})\big]_{+}^{2}}{2(\delta T+1/\zeta)}\Big\}\\
&\le \frac{L}{\zeta}\ln\frac{K}{L}+2(e-2)\zeta KT+\zeta L^2 T+\frac{\zeta LKT}{1-\gamma}+\gamma LT.
\end{align*}
Let $F(T)=\frac{L}{\zeta}\ln\frac{K}{L}+2(e-2)\zeta KT+\zeta L^2 T+\frac{\zeta
  LKT}{1-\gamma}+\gamma LT$. 
Then we have results in the form of (\ref{eq:vio-bound}):
$ {\text{Reg}}_{\tilde{\pi}}(T) \le F(T),\text{ and }
{\text{Vio}}_{\tilde{\pi}}(T) \le \sqrt{2\left(F(T)+LT\right)(\delta
  T+1/\zeta)}+\gamma LT.$ Let $\gamma=\Theta(T^{-\frac{1}{3}})$ and
$\delta=\Theta(T^{-\frac{1}{3}})$. 
Thus, we have $\zeta=\Theta(\frac{1}{K}T^{-\frac{2}{3}})$. 
Finally, we have ${\text{Reg}}_{\tilde{\pi}}(T) \le O(LK\ln(K)T^{\frac23})$ and $
{\text{Vio}}_{\tilde{\pi}}(T) \le O(L^\frac{1}{2}K^\frac{1}{2}T^\frac{5}{6})$.
\end{proof}

 \section{Experiments}
\label{sec:experiments}
In this section, we first examine the web links' multi-level feedback structures
in two real-world datasets from the Kaggle Competitions, \emph{Avito Context Ad
  Clicks}~\cite{kaggle2015avito} and \emph{Coupon Purchase
  Prediction}~\cite{kaggle2016coupon}, referred to as ``\adclick{}'' and
``\coupon{}'' in this paper. 
Then we conduct a comparative study by applying \lexp{} and two state-of-the-art
context-free bandit algorithms, \textsc{CUCB}~\cite{chen2013combinatorial} and
\textsc{Exp3.M}~\cite{uchiya2010algorithms}, to
show the effectiveness of \lexp{} in
links selection. 
\subsection{Multi-level Feedback Structure Discovery}
\label{subsec:structure-discovery}

The \adclick{} data is collected from users of the website
\href{https://www.avito.ru}{Avito.ru}
where a user who is interested in an ad has to first click to view the ad before
making a phone request for further inquiries.
The data involves the logs of the visit stream and the phone request stream of
$71,677,831$ ads. 
We first perform some data cleaning. 
For each ad in \adclick{}, we count the number of views of the ad and
thereafter, count the number of phone requests that ad received. 
In particular, we filter out the ads that have an abnormally large number of views
(greater than $2000$
), and the ads that receive few numbers of phone requests (smaller than $100$
). 
Finally, we obtain $225$ ads from \adclick{}. 
For each of these $225$ ads, we divide the number of phone requests by the
number of views to get the Phone Request Rate. 
We normalize the numbers of views of each ad to the interval $[0,1]$ using
min-max scaling. 
The normalized number of views can be taken as the CTRs of
the ads. 
As such, we find the multi-level feedback structure for each ad in the
\adclick{} data: the CTR corresponds to the ad's \emph{$1$st level feedback},
the Phone Request Rate corresponds to the ad's \emph{$2$nd level feedback}, and
the product of CTR and the Phone Request Rate corresponds to the \emph{compound
  $2$-level feedback}.

The \coupon{} data is extracted from transaction logs of $32,628$ coupons on
the site \href{https://ponpare.jp/}{ponpare.jp}, where users first browse
a coupon and then decide whether to purchase the coupon or not.
We extract $271$ coupons from \coupon{}. 
For each coupon, we divide its number of purchase by the number it was browsed,
and take the ratio as the \emph{Coupon Purchase Rate}. 
We then normalize all the browsed times to $[0,1]$ using min-max scaling and
refer to the normalized browsed times as to the CTRs of the coupons. 
Thus, for each coupon in \coupon{}, the CTR corresponds to the coupon's
\emph{$1$st level feedback}, the Coupon Purchase Rate corresponds to the
coupon's \emph{$2$nd level feedback}, and the product of the CTR and the Coupon
Purchase Rate is the \emph{compound $2$-level feedback}.

Next, we validate our previous claim in Sec.~\ref{sec:model} that the
$1$st level feedback and the $2$nd level feedback in the multi-level feedback
structure do not have a strong correlation. 
In \adclick{}, we find that the CTRs and the Phone Request Rates 
are not strongly correlated with a correlation coefficient $-0.47$.
Similarly, low correlation can also be found in \coupon{}
where the correlation coefficient is
$-0.27$ only.

\begin{figure*}[!thbp]
  \setlength{\abovecaptionskip}{0pt} \setlength{\belowcaptionskip}{-1pt}
  \setlength{\subfigcapskip}{-3pt}
  \centering
  \subfigure[Experiment 1: Cumulative rewards]{
  \includegraphics[width=0.246\textwidth]{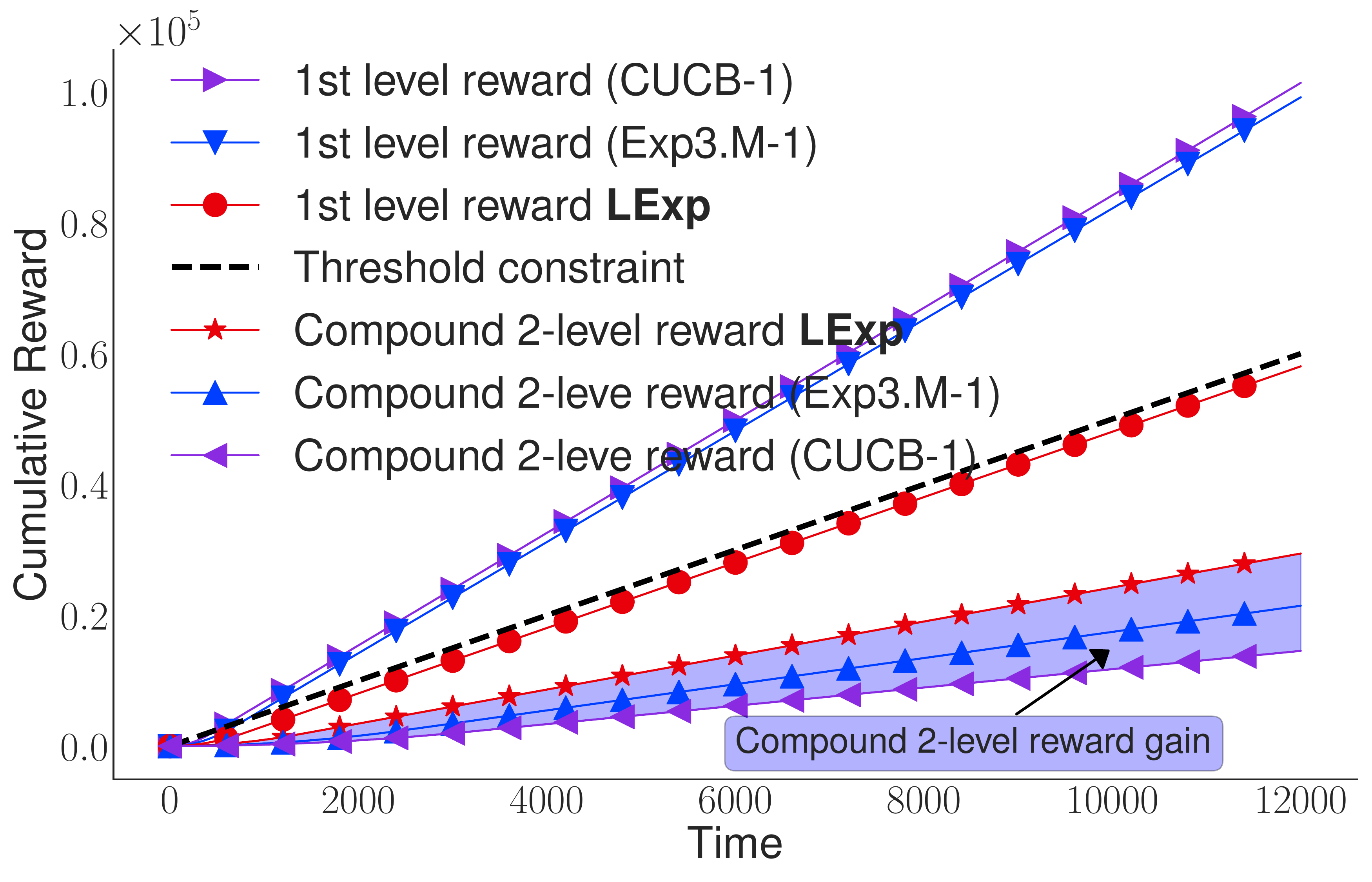}
 \label{fig:cmp-ad:cumreward1}
  }\hspace{-2ex}
  \subfigure[Experiment 2: Regret/Violation]{
  \includegraphics[width=0.246\textwidth]{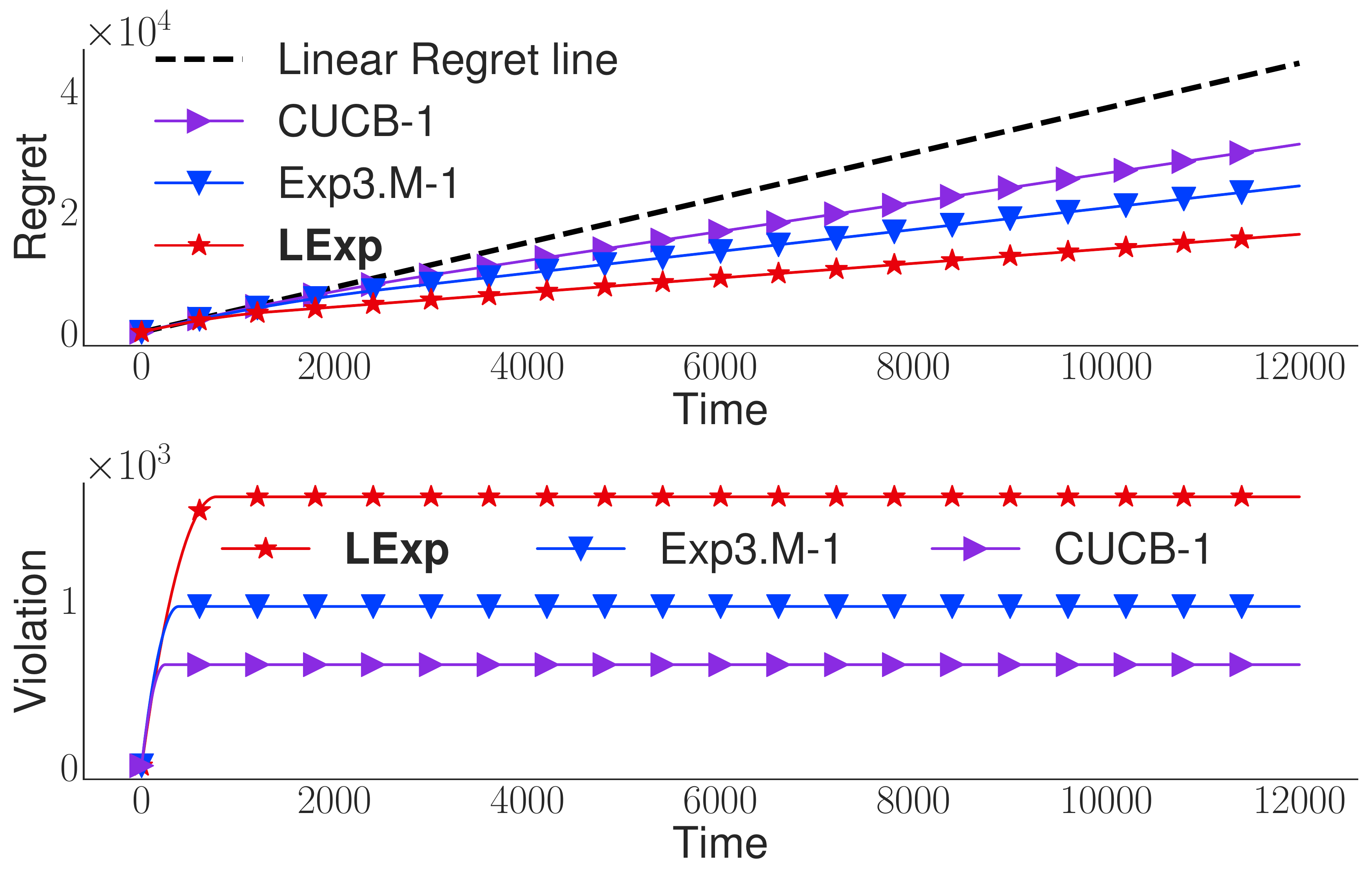}
 \label{fig:cmp-ad:regret-vio1}
  }\hspace{-2ex}
  \subfigure[Experiment 3: Cumulative reward]{
  \includegraphics[width=0.246\textwidth]{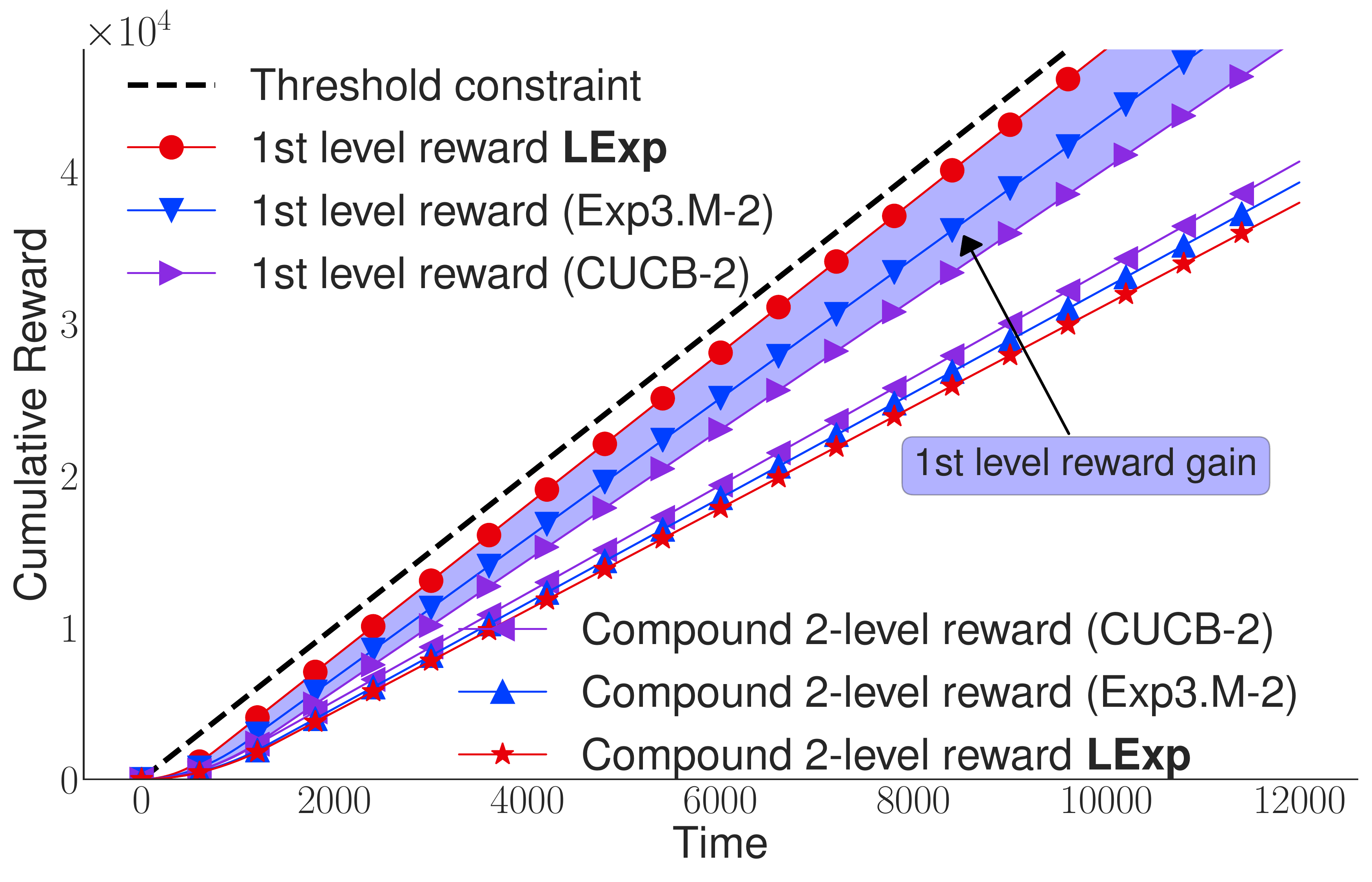}
  \label{fig:cmp-ad:cumreward2}
  }\hspace{-2ex}
  \subfigure[Experiment 4: Regret/Violation]{
  \includegraphics[width=0.246\textwidth]{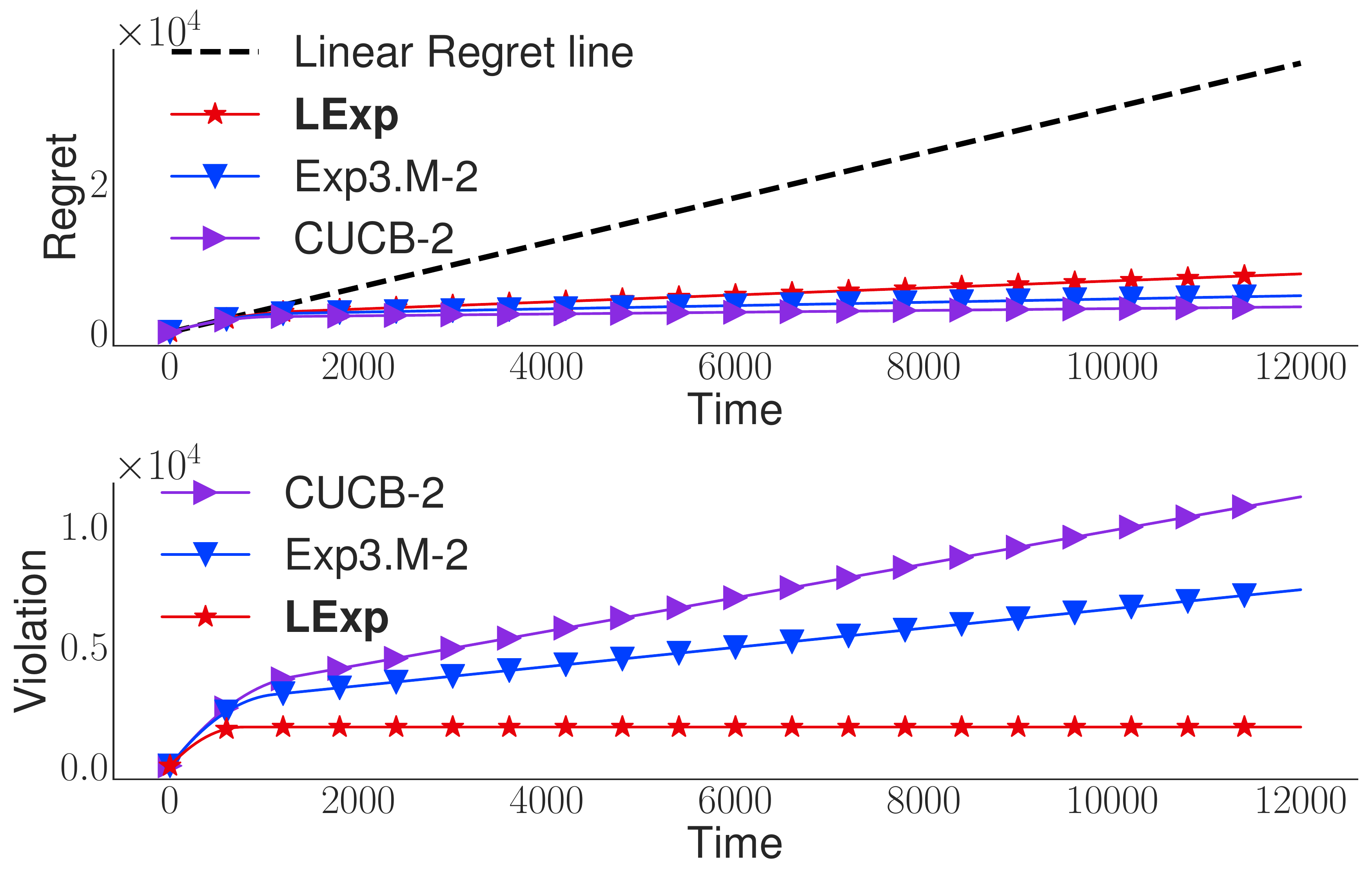}
  \label{fig:cmp-ad:reg-vio2} 
  }\hspace{-2ex}
\caption{\textmd{Four groups of comparative experiments among \lexp{}, \textsc{Exp3.M} and \textsc{CUCB}
                 on \adclick{}. $K=225$, $L=20$, $h=5$ and $T=12000$ ($\gamma=0.019$, $\delta=0.021$). Note that both regrets and violations are calculated in an accumulative fashion.}}
 \label{fig:cmp-ad}
\vspace{-10pt}
\end{figure*}

\subsection{Comparative Study on Links Selection}
\label{sec:performance-cmp-study}

We simulate the multi-level feedback structures in the real-world
datasets and model the probabilistic ads/coupons selection process with
 time-variant rewards.
In particular, for each of the $225$ ads in \adclick{}, we treat its CTR/$1$st
level feedback as a Bernoulli random variable with the mean CTR taking from
\adclick{}, and we vary the Phone Request Rate/$2$nd level reward over time in a
similar fashion as a sinusoidal wave (similar to~\cite{BesbesNIPS2014}): the
$2$nd level reward starts from a random value drawn uniformly from $0$ to the
mean Phone Request Rate taking from \adclick{}; then in each time slot, it
increases or decreases at rate ${10}/{T}$ until reaching the mean or $0$. 
This time-variant rewards can model the seasonal fluctuations of the potential revenue of the ads. 
For each of the $271$ coupons in
\coupon{}, we simulate its $2$-level rewards in the same way.

In our performance comparison, the \textsc{CUCB} algorithm always selects the
top-$L$ arms with the highest UCB (upper confidence bound) indices without
considering any constraint. 
For the \textsc{CUCB} algorithm, on the one hand, if we only want to maximize
the total $1$st level rewards, the $L$ arms of the highest UCB indices $\hat
a_{i}^{t} + \sqrt{3\ln t/(2 N_{i}(t))}$ are selected at each time $t$
(\textsc{CUCB-1}), where $N_{i}(t)$ is the number of times that the arm $i$
has been selected by time $t$. 
On the other hand, if we only want to maximize the total compound $2$-level
rewards, the $L$ arms of the highest UCB indices $\hat g_{i}^{t} + \sqrt{3\ln t/(2
  N_{i}(t))}$ will be selected (\textsc{CUCB-2}).
For \textsc{Exp3.M},
only the $1$st level reward estimation $\bm{\hat a}_{t}$ is considered
when we only maximize the total $1$st level rewards (\textsc{Exp3.M-1}), and
only the compound $2$-level reward estimation $\bm{\hat g}_{t}$ is considered
when we only maximize the total compound $2$-level rewards
(\textsc{Exp3.M-2}). 
In our experiments, the \emph{cumulative} $1$st level reward and \emph{the
  cumulative} compound $2$-level reward at time $t$ are calculated using
$\sum_{t'=1}^{t}\sum_{i\in \mathcal{I}_{t'}}a_{i}^{t'}$ and
$\sum_{t'=1}^{t}\sum_{i\in \mathcal{I}_{t'}}g_{i}^{t'}$, respectively. 
The \emph{regret} at time $t$ is calculated using
$\sum\nolimits_{t'=1}^{t}\bm{g}_{t'}^{\intercal}\bm{x}^{*} -
\sum\nolimits_{t'=1}^{t}\sum_{i\in \mathcal{I}_{t'}}g_{i}^{t'}$ where
$\bm{x}^{*}$ is the optimal probabilistic selection vector by solving
$\max_{\bm{a}^{\intercal}\bm{x}\ge
  h}\sum\nolimits_{t'=1}^{t}\bm{g}_{t'}^{\intercal}\bm{x}$ with $\bm{a}$ taking
from the datasets. 
The \emph{violation} at time $t$ is calculated using
$\sum\nolimits_{t'=1}^{t}(h-\sum_{i\in \mathcal{I}_{t'}}a_{i}^{t'})_{+}$.

For \adclick{}, we run \lexp{}, the \textsc{Exp3.M} variants:
\textsc{Exp3.M-1} and \textsc{Exp3.M-2}, and the \textsc{CUCB} variants:
\textsc{CUCB-1} and \textsc{CUCB-2}. 
The parameter settings are shown in Fig.~\ref{fig:cmp-ad}.

\noindent
{\bf $\bullet$ Experiment 1} (considering total $1$st level rewards only):
In Fig.~\ref{fig:cmp-ad:cumreward1}, we compare the cumulative rewards of
\lexp{}, \textsc{Exp3.M-1} and \textsc{CUCB-1}. 
Specifically, the \emph{cumulative} threshold constraint is a linear function of
$t$, i.e., $h\cdot t$, shown by the black-dash line. 
One can observe that \lexp{} satisfies the threshold constraint
because its slope is equal to $h$ after time $t=841$, meaning that there is no
violation for \lexp{} after $t=841$. 
On the contrary, \textsc{CUCB-1} and
\textsc{Exp3.M-1} initially do not satisfy the
threshold constraint and both exceed the threshold constraint. 
Furthermore, the cumulative compound $2$-level rewards of \textsc{CUCB-1}
 and \textsc{Exp3.M-1} are
both below that of \lexp{} as the ads selected by
\textsc{CUCB-1} and \textsc{Exp3.M-1} with high $1$st level rewards do not
necessarily result in a high $2$nd level rewards. 
Thus, by selecting ads that satisfy the threshold constraint, even when the
cumulative $1$st level rewards is restricted by the threshold, \lexp{} gives the
highest total compound $2$-level rewards with the compound $2$-level rewards
gain shown in the blue shaded area.

\noindent
{\bf $\bullet$ Experiment 2} (Regrets \& Violations on Experiment 1):
Fig.~\ref{fig:cmp-ad:regret-vio1} shows the regrets and
violations of \textsc{CUCB-1}, \textsc{Exp3.M-1} and \lexp{}. 
First, we draw the linear regret line which shows the largest
regret for the ads selection problem where no ads are selected at each time,
i.e., all the rewards are lost. 
The regret of \textsc{CUCB-1} and the regret of
\textsc{Exp3.M-1} are very close to the linear regret
line. 
They are also greater than the regret of \lexp{}, which has a
sub-linear property, and this further confirms the fact that cumulative compound
$2$-level rewards of \textsc{CUCB-1} and \textsc{Exp3.M-1} are both below that
of \lexp{} in Fig.~\ref{fig:cmp-ad:cumreward1}. 
For the violations, all the three algorithms, \lexp{}, \textsc{Exp3.M-1}
 and \textsc{CUCB-1} first
show some increases and then remain constant. 
This is because they all are in the {\em exploration phase}: they first select
ads with random $1$st level rewards when the $1$st level rewards are still
unknown, and later select ads with high $1$st level rewards that satisfy or
exceed the threshold constraint. 
\lexp{} is less
aggressive than \textsc{Exp3.M-1} and \textsc{CUCB-1} which both select the ads
with total $1$st level rewards that far exceed the threshold. 
But as we can observe in Fig.~\ref{fig:cmp-ad:cumreward1}, 
\lexp{} performs much better than these algorithms. In summary, \lexp{} takes longer to explore the optimal ads but after some
trials, the violation \emph{at each time} diminishes to zero. This can be
observed from Fig.~\ref{fig:cmp-ad:regret-vio1} since the violations remains unchanged after about $850$ time slots.

\noindent
{\bf $\bullet$ Experiment 3} (considering the compound $2$-level rewards only):
Fig.~\ref{fig:cmp-ad:cumreward2} shows the cumulative rewards of \lexp{},
\textsc{Exp3.M-2} and \textsc{CUCB-2}.
Specially, the cumulative compound $2$-level rewards of \textsc{CUCB-2}
and \textsc{Exp3.M-2} are
both larger than that of \lexp{} as they consider maximizing the
total compound $2$-level rewards only. 
However, the cumulative $1$st level rewards of both \textsc{Exp3.M-2}
and \textsc{CUCB-2} increase slower than the threshold
constraint $h\cdot t$ as their slopes are less than $h$. 
This means that they violate the constraint and the gap between their cumulative
$1$st level rewards and the threshold constraint continues to grow as time goes
by. 
For the cumulative $1$st level rewards of \lexp{}, the slope
increases up to $h$ and maintains at $h$ and therefore the gap becomes a
constant. 
In summary, \lexp{} ensures that the threshold constraint is satisfied and
produces the additional $1$st level reward gain (blue-shaded area) compared with
\textsc{Exp3.M-2} and \textsc{CUCB-2}.

\noindent
{\bf $\bullet$ Experiment 4} (Regrets \& Violations on Experiment 3):
Fig.~\ref{fig:cmp-ad:reg-vio2} shows the regrets and violations
of \lexp{}, \textsc{Exp3.M-2} and \textsc{CUCB-2}. 
For the regrets, the regrets of \textsc{Exp3.M-2},
\textsc{CUCB-2} and \lexp{} are all
{\em sub-linear} and below the linear regret line, as these
three algorithms all aimed at minimizing the regret. 
Among them, \lexp{} has a comparable regret but it also has an addition
property, which is to satisfy the threshold constraint. 
As for the violation, the violations of both \textsc{CUCB-2} and \textsc{Exp3.M-2} end up linear as
their cumulative $1$st level rewards increase slower than the cumulative
threshold constraint as shown in Fig.~\ref{fig:cmp-ad:cumreward2}. 
In contrast, the violation of \lexp{} increases and then
eventually stays constant.
This confirms that \lexp{} aims to satisfy the constraint and it will not make
mistake after some rounds of learning.
This confirms that \lexp{} aims to satisfy the constraint and it will not make
mistake after some rounds of learning.

For \coupon{}, we obtain similar experimental results and we can draw similar
conclusions. Therefore,
we omit the detailed descriptions for conciseness. In
summary,
our experimental results show that \lexp{} is the only algorithm which balances
the regret and violation in the ads/coupons selection problem.

\section{Related Work}
\label{sec:related-work}

One common approach to the links selection problem is to perform A/B
testing~\cite{Deng2017WSDM}, which splits the traffic for
different sets of links on two different web pages, and then evaluate their
rewards. 
However, A/B testing does not have any loss/regret guarantee as it splits equal
amounts of traffic to the links regardless of the links' rewards. 
That said, A/B testing is still widely used in commercial web systems. 
Our algorithm can be viewed as a complementary approach to A/B testing, e.g.,
our algorithm can select the set of links with the $1$st level reward above a
given threshold and facilitate a 
more efficient A/B testing for the
links selection problem.

Another approach is to model the links selection problem as contextual bandit
problems.
\cite{li2010contextual} first formulated a contextual bandit problem aiming at
selecting articles/links that maximize the total number of clicks based on the
user-click feedback. 
Recently, \cite{collfilteringSIGMETRICS16} and~\cite{CFBanditSIGIR16}
incorporated the collaborative filtering method into contextual bandit
algorithms using users' profiles to recommend web links. 
However, contextual information is not always available due to cold
start~\cite{elahi2016survey} or blocking of cookie
tracking~\cite{meng2016trackmeornot}. 
Moreover, contextual bandit problem formulations neglect the multi-level
feedback structures of the links and do not consider any constraint.

Our bandit formulation is related to the bandit models with multiple plays, where
multiple arms are selected in each round. 
\cite{uchiya2010algorithms} presented the \textsc{Exp3.M} bandit
algorithm that extends the single-played \textsc{Exp3}
algorithm~\cite{Auer:2003:nonsto} to multiple-played cases using exponential
weights.
\cite{chen2013combinatorial}
proposed an algorithm that selects multiple arms with the highest upper
confidence bound (UCB) indices. 
\cite{Komiyama1506} presented the multiple-play Thompson Sampling
algorithm (MP-TS) for arms with binary rewards.
\cite{Vorobev2015WWW}
proposed a bandit-based ranking algorithm for ranking search queries. 
Our bandit model differs from these bandit models as we further consider the
constraint on the total $1$st level rewards in selecting the multiple arms.

Note that the constraint in our constrained bandit model is very different from
that in bandit with budgets~\cite{dengbanditICDM07,xia2016budgeted} and 
bandit with knapsacks~\cite{agrawal2016efficient}.
For these works, the optimal stopping time is considered since no arms can be
selected/played if the budget/knapsacks constraints are violated. 
However, the constraint in our model does not pose such restrictions and the
arm selection procedure can continue without stopping. 
Finally, our constrained bandit problem is related but different from the bandit
model considered in~\cite{mahdavionline} which tries to balance regret and
violation.
They only considered selecting a single arm without any multi-level
rewards. 
While in our work, we consider how to select multiple arms and each arm is
associated with multi-level rewards, making our model more challenging and
applicable to the web links selection problem.
 \section{Conclusion}
\label{sec:conclusion-futurework}
In this paper, we reveal the intrinsic multi-level feedback structures of web
links and formulate the web links selection problem. 
To our best knowledge, we are the first to model the links selection problem
with multi-level feedback structures as a stochastic constrained bandit problem.
We propose and design an effective links selection algorithm \lexp{} with
{\em provable sub-linear regret} and {\em violation} bounds. 
Furthermore, we demonstrate how to learn/mine the multi-level reward structures
of web links in two real-world datasets. 
We carry out extensive experiments to compare \lexp{} with the
state-of-the-art context-free bandit algorithms and 
show that \lexp{} is superior in selecting web links
with constrained multi-level feedback by balancing both regret and
violation. 
\ignore{For future work, we will consider reducing the complexity when the
  ranking of the links and the dependence among the links are involved, and
  designing an anytime bandit algorithm that eliminates the dependence on the
  prior knowledge of $T$ with even lower regret and violation bounds.} \bibliographystyle{IEEEtran}

\end{document}